\documentclass[10pt,conference]{IEEEtran}

\usepackage{amsmath}
\usepackage{amssymb}
\usepackage{stmaryrd}
\usepackage[T1]{fontenc}
\usepackage{url}
\usepackage{bm}
\newcommand{\okra}[1]{\left( #1 \right)}

\newcommand{\sred}[1]{\textbf{E}\, #1}
\newcommand{\ceil}[1]{\lceil #1 \rceil}
\newcommand{\floor}[1]{\lfloor #1 \rfloor}
\newcommand{\kwad}[1]{\left[ #1 \right]}
\newcommand{\klam}[1]{\left\{ #1 \right\}}

\newcommand{\boole}[1]{{\bf 1}_{\klam{#1}}}

\DeclareMathOperator{\CUT}{CUT}

\newtheorem{conjecture}{Conjecture}[section]
\newtheorem{definition}[conjecture]{Definition}
\newtheorem{example}[conjecture]{Example}

\newtheorem{theorem}[conjecture]{Theorem}
\newtheorem{corollary}[conjecture]{Corollary}
\newtheorem{lemma}[conjecture]{Lemma}
\newtheorem{question}[conjecture]{Question}

\begin{document}

\title{The Redundancy of a~Computable Code \\ on a~Noncomputable Distribution}

\author{
\authorblockN{{\L}ukasz D\k{e}bowski}
\authorblockA{Centrum Wiskunde \& Informatica\\
  Science Park 123, NL-1098 XG Amsterdam\\ 
  The Netherlands\\
  Email: debowski@cwi.nl}
}

\maketitle

\begin{abstract}
  We introduce new definitions of universal and super\-universal
  computable codes, which are based on a~code's ability to approximate
  Kolmogorov complexity within the prescribed margin for all
  individual sequences from a~given set. Such sets of sequences may be
  singled out almost surely with respect to certain probability
  measures.

  Consider a~measure parameterized with a~real parameter and put an
  arbitrary prior on the parameter. The Bayesian measure is the
  expectation of the parameterized measure with respect to the prior.
  It appears that a~modified Shannon-Fano code for any computable
  Bayesian measure, which we call the Bayesian code, is
  super\-universal on a~set of parameterized measure-almost
  all sequences for prior-almost every parameter.

  According to this result, in the typical setting of mathematical
  statistics no computable code enjoys redundancy which is ultimately
  much less than that of the Bayesian code. Thus we introduce another
  characteristic of computable codes: The catch-up time is the length
  of data for which the code length drops below the Kolmogorov
  complexity plus the prescribed margin.  Some codes may have smaller
  catch-up times than Bayesian codes.
\end{abstract}

\IEEEpeerreviewmaketitle

\section{What is a~good computable code?}
\label{secGoodComputable}

Giving a~reasonable definition to the notion of a~good general-purpose
compression algorithm is very important. Not so much for the practical
data compression but rather for a~theoretical analysis of statistical
inference and machine learning. All parameter estimation or prediction
algorithms can be transformed into compression algorithms via the idea
of the plug-in code \cite{Dawid84,Rissanen84}, see also \cite[Section
6.4.3]{Grunwald07}.
A transformation in the opposite direction can be done for prediction,
with a~guaranteed standard risk in the iid case \cite[Proposition
15.1--2]{Grunwald07}.  With certain restrictions, the better we
compress the better we predict.

This article proposes a~new simple theoretical framework for
computable universal compression of random data (and thus for their
prediction). Our results lie between the idealized algorithmic
statistics \cite{VitanyiLi00,GacsTrompVitanyi01} and the present MDL
perspective on mainstream statistical inference
\cite{BarronRissanenYu98,Grunwald07,ErvenGrunwaldRooij08}. We offer
a~clearer path to understanding what good compression procedures are
when the predicted data are generated by very complicated probability
measures, cf.\ \cite{Debowski08c}.

The prefix Kolmogorov complexity $K(x)$ is the length of a~code for
a~string $x$ which we can never beat more than by a~constant when we
use computable prefix codes \cite[Chapter 3]{LiVitanyi97}.\footnote{To
  fix our notation, the prefix code $C:\mathbb{X}^+\rightarrow
  \mathbb{Y}^+$ encodes strings over a countable alphabet $\mathbb{X}$
  as strings over a~finite alphabet $\mathbb{Y}=\klam{0,1,...,D-1}$
  and $\log$ is the logarithm to the base $D$. The prefix Kolmogorov
  complexity is considered with respect to the computer which accepts
  programs only from a~prefix-free subset of $\mathbb{Y}^+$, $\sum_{x}
  D^{-K(x)}<\Omega<1$.  We call code $C$ computable if both $C$ and
  the inverse mapping $C^{-1}$ can be computed by the computer.
  $|C(x)|$ is the length of $C(x)$.}  Consequently, our theoretical
evaluation of compression algorithms will be based on Kolmogorov
redundancy $|C(x)|-K(x)$ rather than on the traditional Shannon
redundancy
\begin{align}
  \label{Redundancy}
  |C(x)|+\log P(x),
\end{align}
where $C$ is the inspected computable prefix code and
$P\in\mathcal{M}$ is one of many candidate distributions for the data.

A~large body of literature has been devoted to studying codes
that are minimax optimal with respect to (\ref{Redundancy}), exactly
\cite{Topsoe79,Shtarkov87en2,Haussler97} or asymptotically
\cite{BarronRissanenYu98,Grunwald07}.  Let us notice that if the
minimax expected Shannon redundancy
\begin{align}
  \label{MinimaxRed}
  \min_{C} \sup_{P\in\mathcal{M}} \sred_{x\sim P} \kwad{|C(x)|+\log P(x)}
\end{align}
or the minimax regret
\begin{align}
  \label{MinimaxReg}
  \min_{C} \sup_{P\in\mathcal{M}} \max_{x} \kwad{|C(x)|+\log P(x)}
\end{align}
are finite, plausibly bounded in terms of the data length, and
achieved by a~unique code $C$ then the corresponding minimax
properties appear a~plausible rationale to argue for code $C$'s
optimality against data typical of a~class of distributions
$\mathcal{M}$.

Things change when (\ref{MinimaxRed}) or (\ref{MinimaxReg}) are
infinite since then every code is a~minimizer.  Infinite or unbounded
minimax values appear in fact in many statistical models: (i) There
are no universal redundancy rates for stationary ergodic processes
\cite{Shields93}. (ii) Even in the parametric iid case, like Poisson
or geometric, one often has to restrict the parameter range to
a~compact subset to have a~reasonable minimax code \cite[Theorem 7.1
and Sections 11.1.1--2]{Grunwald07}. In a~surprising contrast, the
redundancy for computable parameters can be very small, which is
known as superefficient estimation/compression
\cite{LeCam53,Vovk91en2,BarronHengartner98}.

The minimax values (\ref{MinimaxRed}) or (\ref{MinimaxReg}) may be
infinite because there is no worst case of data rather than no
intuitively good code. Often there exists an intuitively good code but
to single it out with the minimax criterion, we have to modify the
score (\ref{Redundancy}) with some penalty. This idea has emerged in
the MDL statistics in recent years. Gr\"unwald \cite[Sections 11.3 and
11.4]{Grunwald07} reviewed a~bunch of proposed heuristic penalties,
which he called the ``luckiness functions'' or conditional NML
(normalized maximum likelihood). In general, the penalties have form
$\ell(P,x)$ so the mini\-maxi\-mized function is
\begin{align}
  \label{Penalty}
  |C(x)|+\log P(x)-\ell(P,x)
\end{align}

Now, an important simple new idea. Typically for mathematical
statistics, $P$ is noncomputable (in the absolute sense).  For
instance, it may be given by an analytic formula with an
algorithmically random parameter, to be estimated from the observed
data $x$ rather than known beforehand. On the other hand, the code $C$
that we are searching for must be computable. We owe this insight to
Vovk \cite{Vovk09}, who writes:
\begin{quote}
  The purpose of estimators is to be used for computing estimates, and
  so their computability is essential. Accordingly, in our discussion
  we restrict ourselves to computable estimators.  

  A parameter point is not meant to be computed by anybody. Depending
  on which school of statistics we listen to, it is either a constant
  chosen by Nature or a mathematical fiction.
\end{quote}

Consequently, the baseline $-\log P(x)$ in the coding game
(\ref{Redundancy}) should be replaced by \emph{something} uniformly
closer to the smallest code length that we can achieve by effective
computation. The prefix Kolmogorov complexity $K(x)$ seems a~fortunate
candidate since
\begin{align}
  \label{Why}
  |C(x)|\ge K(x)-\tilde K(C^{-1})
  ,
\end{align}
where $\tilde K(C^{-1})$ is the length of any program to decode $C$,
i.e., to compute $C^{-1}$. When designing a~general-purpose compressors
$C$, one usually wants to keep $\tilde K(C^{-1})$ small.

We should subtract the generic luckiness function
\begin{align}
  \label{Luckiness}
  \ell(P,x):=K(x)+\log P(x)
\end{align}
from the criterion (\ref{Redundancy}) before the minimax is applied
since otherwise we punish an intuitively good code $C$ for unlearnable
idiosyncrasies and nonuniformity of the data.  The luckiness
(\ref{Luckiness}) does not depend on code $C$ and its expectation is
nonnegative.  As we will elaborate in Section \ref{secLuckiness},
this very $\ell(P,x)$ is close in several senses to \emph{algorithmic
  information} $I(P:x)$ about $x$ in $P$.

We conjecture that $I(P:x)$ can grow for noncomputable $P$ very fast
in terms of the data length $|x|$, like any function $o(|x|)$ even in
the iid case, cf.\ \cite{Debowski08c}.  The order of the growth
depends not only on the ``parametric class'' of $P$ that statisticians
like to think of but also on the exact ``displacement'' of algorithmic
randomness in the possibly infinite definition of $P$. For instance,
if $P$ is computable given a~computable parameter value then $I(P:x)$
is bounded by the finite Kolmogorov complexity of $P$ in view of the
symmetry of algorithmic information \cite[Theorem
3.9.1]{LiVitanyi97}. This bound can be also associated with the
existence of a computable superefficient estimator of the parameter
\cite{BarronHengartner98,Vovk91en2,Vovk09}.

Although $K(x)$ is noncomputable and we cannot evaluate the value of
$K(x)$ for any particular string $x$, we can obtain sufficiently good
estimates of Kolmogorov complexity for strings typical of certain
probability measures.  This observation inspires our new
individualistic definitions of universal and superuniversal codes,
which avoid minimax whatsoever. In the following, italic
$x,y,...\in\mathbb{X}^+$ are strings (of finite length), boldface
$\bm{x},\bm{y},...\in\mathbb{X}^\infty$ are infinite sequences, and
calligraphic $\mathcal{X},\mathcal{S},...\subset \mathbb{X}^\infty$
are subsets of these sequences. Symbol ${x}_n$ denotes the $n$-th
symbol of $\bm{x}$ and ${x}^n$ is the prefix of $\bm{x}$ of length
$n$: $\bm{x}={x}_1{x}_2{x}_3...$, ${x}^n={x}_1{x}_2...{x}_n$.
Consequently:
\begin{definition}[universal codes]
  \label{defiUni} 
  Code $C$ is called $(\mathcal{X},o(f(n)))$-universal if it is
  a~computable prefix code and $\lim_{n\rightarrow\infty}
  \kwad{|C({x}^n)|-K({x}^n)}/f(n)=0$ holds for all
  $\bm{x}\in\mathcal{X}$.
\end{definition}
\begin{definition}[super\-universal codes]
  \label{defiUUni} 
  Code $C$ is called $(\mathcal{X},f(n))$-super\-universal if it is
  a~computable prefix code and $|C({x}^n)|-K({x}^n)< f(n)$
  holds $n$-ultimately for all  $\bm{x}\in\mathcal{X}$.
\end{definition}
Phrase ``$n$-ultimately'' is an abbreviation of ``for all but finitely
many $n\in\mathbb{N}$''.  

Although Definitions \ref{defiUni}--\ref{defiUUni} reinterpret several
probabilistic concepts of code universality that have been
contemplated by Gr\"unwald \cite[pages 183, 186, and 200]{Grunwald07},
only two specific kinds of known codes fall under these definitions.

The codes discovered firstly are $(\mathcal{S},o(n))$-universal codes
for sequences typical of certain stationary measures, such as the
$\text{LZ}$ code and many similar
\cite{ZivLempel77,NeuhoffShields98,Debowski08c,Kieffer78}.
Namely, for each stationary probability measure $\bm{P}$ over a~finite
alphabet there exists a~set $\mathcal{S}_{\bm{P}}$ of infinite sequences
such that $\bm{P}(\mathcal{S}_{\bm{P}})=1$ and $\text{LZ}$ is
$(\mathcal{S}_{\bm{P}},o(n))$-universal.\footnote{Our notation for
  distributions and measures follows the distinction between strings
  and infinite sequences.  Italic $P$ is a~distribution of countably
  many strings $x$ with $P(x)\ge 0$ and $\sum_x P(x)=1$.  Boldface
  $\bm{P}$ is also a~distribution of strings $x$, $\bm{P}(x)\ge 0$,
  but normalized against strings of fixed length $\sum_{x} \bm{P}(x)
  \boole{|x|=n}=1$ and satisfying the consistency condition $\sum_y
  \bm{P}(xy)\boole{|y|=n}=\bm{P}(x)$.  Consequently there is a~unique
  measure on the measurable sets of infinite sequences $\bm{x}$, also
  denoted as $\bm{P}$, such that $\bm{P}(\klam{\bm{x}:{x}^n=x\text{
      for }n=|x|})=\bm{P}(x)$.%
} Consequently, we may put $\mathcal{S}=\bigcup_{\bm{P}\in\mathbb{S}}
\mathcal{S}_{\bm{P}}$, where $\mathbb{S}$ is the set of all such
measures. 

There exists also a~second kind of good codes which consists of
super\-universal codes for sequences typical of computable
measures. For each computable measure $\bm{P}$ there exists a~specific
set $\mathcal{B}_{\bm{P}}$ of infinite sequences such that
$\bm{P}(\mathcal{B}_{\bm{P}})=1$ and a~simple modification of the
computable Shannon-Fano code is
$(\mathcal{B}_{\bm{P}},|{c}(n)|+1)$-super\-universal.\footnote{We
  use symbol ${c}:\mathbb{N}\rightarrow \mathbb{Y}^+$ to denote
  a~computable prefix code for natural numbers, $\sum_{n}
  D^{-|{c}(n)|}\le 1$. For example, ${c}(n)$ may be chosen as
  the recursive $\omega$-representation for $n$ \cite{Elias75}. Then
  $|{c}(n)|=\log^*n+1$, where $\log^*n$ is the iterated logarithm
  of $n$ to the base $D$.  A~different ${c}(n)$ may be convenient
  for a~study of superefficient compression,
  cf. \cite{BarronHengartner98}.  By an analogy to the distinction
  between $P$ and $\bm{P}$, we propose symbol $\bm{C}$ to denote
  a~system of computable prefix codes for strings of fixed length. The
  corresponding Kraft inequalities are $\sum_{x} D^{-|\bm{C}(x)|}
  \boole{|x|=n}\le 1$ versus $\sum_{x} D^{-|C(x)|}\le 1$.  Each code
  of form ${c}(|x|)\bm{C}(x)$ is a~prefix code for strings of any
  length but the converse is not true.} In the case of $\bm{P}(x)=\int
\bm{P}_{\bm{\theta}}(x) d\bm{\pi}(\bm{\theta})$, we call this code the
Bayesian code with respect to
$(\klam{\bm{P}_{\bm{\theta}}},\bm{\pi},{c})$.

Consider the case when $\bm{P}$ is computable whereas
$\bm{P}_{\bm{\theta}}$ is not necessarily so. We will see easily in
Section \ref{secBayesianCode} that
$\bm{P}_{\bm{\theta}}(\mathcal{B}_{\bm{P}})=1
=\bm{P}(\mathcal{B}_{\bm{P}})$ for $\bm{\pi}$-almost all
$\bm{\theta}$. This simple statement establishes in fact the ultimate
near-optimality of Bayesian codes with respect to
$(\klam{\bm{P}_{\bm{\theta}}},\bm{\pi},{c})$ also for data typical of
many simple noncomputable probability measures. The statement appears
very powerful since we can let $\bm{P}_{\bm{\theta}}$ be any
parameterized measures considered by statisticians for years.  To
mention a~few examples, we may consider iid Bernoulli, Poisson or
discretized long-range dependent Gaussian time series.  The result
also explains why the MDL statistics has so resembled Bayesian
inference so far.

The motivation for Bayesian codes in the MDL statistics lies in the
concept of the shortest effective description rather than in beliefs.
Thus, in the MDL paradigm we can go farther and ask what computable
codes are significantly shorter than a~fixed Bayesian
code.\footnote{We consider here only computable Bayesian inference. It
  has been known that $K(x)$ equals the length of certain
  noncomputable code having a~Bayesian interpretation \cite[Example
  4.3.3 and Theorem 4.3.3]{LiVitanyi97}.}  Because of the $K(n)$-high
oscillations of Kolmogorov complexity \cite[Sections 2.5.1 and
3.4]{LiVitanyi97}, one may hardly expect that there exist
$(\mathcal{X},f(n))$-super\-universal codes for
$f(n)=o(K(n))+O(1)$. 

The ultimate redundancy does not seem a~performance score that can be
improved on if we can only define a~computable Bayesian code for the
contemplated statistical problem.  This notwithstanding, another
performance score can be attacked.
\begin{definition}[catch-up time]
  \label{defiCatchUp} 
  The catch-up time for an $(\mathcal{X},f(n))$-super\-universal code
  $C$ is the function
  $\CUT(\,\cdot\,;C):\mathbb{X}^\infty\rightarrow\mathbb{N}\cup\klam{\infty}$
  defined as
  \begin{align*}
    \CUT(\bm{x};C):= 
    \sup \klam{n\in\mathbb{N}: |C({x}^n)|-K({x}^n)\ge f(n)}
    .
  \end{align*}
\end{definition}
\null The catch-up time is the minimal length of data for which the
code becomes almost as good as the Kolmogorov complexity.  A~simple
lower bound for the catch-up time can be obtained by comparing two
computable codes experimentally.  Basing on the data provided by
\cite{ErvenGrunwaldRooij08}, we conjecture that some codes have much
smaller catch-up times than Bayesian codes.

In the remaining part of this article, we detail the mentioned
results. In Section \ref{secLuckiness}, we argue that the generic
luckiness function is close to algorithmic information. In Section
\ref{secBayesianCode}, we prove that Bayesian codes are
super\-universal for data typical of almost all parameter values.
Some ideas for future research are sketched in the concluding Section
\ref{secConclusion}.

Our framework differs in several points to what has been done in the
algorithmic and MDL statistics. Firstly, we insist on computable codes
but apply both Kolmogorov complexity and noncomputable probability
measures to evaluate the quality of the code.  Secondly, we apply
a~stronger version of Barron's ``no hypercompression'' inequality to
upper bound the code length in question with the Kolmogorov complexity
rather.  Secondly, we apply a~stronger version of Barron's ``no
hypercompression'' inequality to upper bound the code length in
question with the Kolmogorov complexity. So far Barron's inequality
was only used to lower bound the code length with minus
log-likelihood.

\section{A generic ``luckiness'' function}
\label{secLuckiness}

We will argue in this section that the generic luckiness function
$\ell(P,x):=K(x)+\log P(x)$ is close to the algorithmic information about
$x$ in $P$.  First of all, let us recall necessary concepts:
\begin{enumerate}
\item The universal computer is a finite state machine that interacts
  with one or more infinite tapes on which only a~finite number of
  distinct symbols may be written in each cell. For convenience, we
  allow three tapes: tape $\alpha$ on which a~finite program is
  written down, tape $\beta$ (oracle) on which an infinite amount of
  additional information can be provided before the computations are
  commenced, and tape $\gamma$ from which the output of computations
  is read once they are finished. We assume that programs which the
  computer accepts on tape $\alpha$ form a~prefix-free set of strings.
\item To compute strings over an alphabet that is larger (e.g.\
  countably infinite) than the alphabet allowed on tape $\gamma$, we
  assume that the contents of $\gamma$ is sent to a~fixed decoder once
  the computations are finished.
\item The prefix Kolmogorov complexity $K(x)$ of a~string $x$ is the
  length of the shortest program on tape $\alpha$ to generate the
  representation of string $x$ on tape $\gamma$ when the computer does
  not read from tape $\beta$.
\item The conditional prefix Kolmogorov complexity $K(x|y)$ is the
  length of the shortest program on tape $\alpha$ to generate the
  representation of string $x$ on tape $\gamma$ when the
  representation of object $y$ is given on tape $\beta$.
\item The representation of an arbitrary distribution $P$ on tape
  $\beta$ is a~list of probabilities $\floor{P(x)D^{m}}D^{-m}$
  discretized up to $m$ digits, enumerated for all strings $x$ and all
  precision levels $d$.  (The same applies to a~measure $\bm{P}$
  respectively.)
\item If the function $(x,d)\mapsto \floor{P(x)D^{m}}D^{-m}$ can be
  computed by a~program then we put $K(P)$ to be the length of the
  shortest such program and call $P$ computable. If $P$ is not
  computable, we let $K(P):=\infty$.
\end{enumerate}

The old idea of Shannon-Fano coding \cite[Section
5.9]{CoverThomas91} yields thus the following proposition:
\begin{theorem}
  \label{theoKP}
  \cite[the proof of Lemma II.6]{GacsTrompVitanyi01} For
  a~computer-dependent constant $A$,
  \begin{align}
    \label{ShannonFano}
    K(x|P) +\log P(x) &\le A
    ,
    \\
    \label{Channel}
    \sred_{x\sim P}\kwad{K(x|P) +\log P(x)} &\ge 0
    .
  \end{align}
\end{theorem}
Constant $A$ is the length of any program on tape $\alpha$ which
computes $x$ given the mapping $y\mapsto P(y)$ put on tape $\beta$ and
$x$'s Shannon-Fano codeword of length $\ceil{-\log P(x)}$
appended on tape $\alpha$ after the program.  Inequality
(\ref{Channel}) is the noiseless coding theorem for entropy and an
arbitrary prefix code. 

The version of (\ref{ShannonFano}) for measure $\bm{P}$ requires an
additional term to identify the string length. Now constant $A$
becomes the length of a~program on tape $\alpha$ which computes $x$
given the mapping $y\mapsto \bm{P}(y)$ put on tape $\beta$, the
prefix-free representation of the string length $n=|x|$ appended on
tape $\alpha$ after the program, and $x$'s Shannon-Fano codeword of
length $\ceil{-\log \bm{P}(x)}$ appended on tape $\alpha$ after that.
As the prefix-free representation of $n$, we choose the shortest
program to generate $n$. The length of this program is denoted as
$K(n)$. For any computable code ${c}$ for natural numbers, we have also
$$K(n)\le \tilde K({c}^{-1})+|{c}(n)|,$$ where 
$\tilde K({c}^{-1})$ is the length of any program to decode ${c}$.
\begin{theorem}
  \label{theoKPB}
  For a~computer-dependent constant $A$,
  \begin{align}
    \label{ShannonFanoB}
    K({x}^n|\bm{P}) +\log \bm{P}({x}^n) &\le A+ K(n)
    ,
    \\
    \label{ChannelB}
    \sred_{\bm{x}\sim \bm{P}}
    \kwad{K({x}^n|\bm{P}) +\log \bm{P}({x}^n)} &\ge 0
    .
  \end{align}
  Moreover, 
  \begin{align}
    \label{BarronB}
    K({x}^n|\bm{P}) +\log \bm{P}({x}^n) &> 0
  \end{align}
  $n$-ultimately for $\bm{P}$-almost all
  sequences $\bm{x}$.
\end{theorem}

Inequality (\ref{BarronB}) stems from a~bit stronger version of
Barron's inequality than given in \cite[Theorem 3.1]{Barron85b}:
\begin{lemma}[Barron's ``no hypercompression'' inequality]
  \label{theoBarron}
  Let $W$ be a prefix code for strings of any length, not necessarily
  computable. Then
  \begin{align}
    \label{Barron}
    |W({x}^n)| +\log \bm{P}({x}^n) &> 0
  \end{align}
  $n$-ultimately for $\bm{P}$-almost all
  sequences $\bm{x}$.
\end{lemma}
\emph{Remark:} We may put $|W(x)|:=K(x|\text{anything fixed})$ or
$|W(x)|:=K(m)+K(x|f(m))$, where $m$ depends on $x$ in whatever way.

\begin{proof}
  Consider function $Q(x)=D^{-|W(x)|}$.  By the Markov inequality,
  \begin{align*}
    \bm{P}\okra{\text{(\ref{Barron}) is false}}
    &=
    \bm{P}\okra{\frac{Q({x}^n)}{\bm{P}({x}^n)}\ge 1}
    \\
    &\le \sred_{\bm{x}\sim \bm{P}} \kwad{\frac{Q({x}^n)}{\bm{P}({x}^n)}} 
    = \sum_x \boole{|x|=n} Q(x)
    .
  \end{align*}
  Hence $\sum_n \bm{P}\okra{\text{(\ref{Barron}) is false}}\le \sum_x
  D^{-|W(x)|}\le 1<\infty$ by the Kraft inequality.  In the following,
  we derive the claim with the Borel-Cantelli lemma.
\end{proof}
 
Let us recall that the algorithmic information about $x$ in $P$ is
\begin{align}
  I(P:x):=K(x)-K(x|P)\ge 0
\end{align}
\cite[Definition 3.9.1]{LiVitanyi97}---the last inequality holds
without any additive constant for our definition of universal
computer. A~bit different definition of symmetric algorithmic
information $I(x;y)$ is sometimes also convenient \cite[Eq.\
II.3]{GacsTrompVitanyi01}. As a~corollary of Theorems \ref{theoKP} and
\ref{theoKPB}, we obtain bounds for luckiness term (\ref{Luckiness})
which read
\begin{align*}
  \ell(P,x)-I(P:x) &\le A
  ,
  \\
  \sred_{x\sim P}\kwad{\ell(P,x)-I(P:x)} &\ge 0
  ,
  \\
  \ell(\bm{P},{x}^n)-I(\bm{P}:{x}^n) &\le A+K(n)
  ,
  \\
  \sred_{\bm{x}\sim \bm{P}}\kwad{\ell(\bm{P},{x}^n)-I(\bm{P}:{x}^n)} &\ge 0
  ,
\end{align*}
whereas
  $\ell(\bm{P},{x}^n)-I(\bm{P}:{x}^n) > 0$
$n$-ultimately for $\bm{P}$-almost all
sequences $\bm{x}$.

\section{Bayes is optimal for almost all parameters}
\label{secBayesianCode}

Adjust the programs for computing $x$ from its Shannon-Fano codeword
so that they use a~built-in subroutine for computing $x\mapsto P(x)$
written on tape $\alpha$ rather than read the definition of this
mapping from tape $\beta$. Then we have:
\begin{theorem}
  \label{theoKPU}
  \cite[Theorem 8.1.1]{LiVitanyi97} For a~computer-dependent constant
  $A$,
  \begin{align}
    \label{ShannonFanoU}
    K(x) +\log P(x) &\le A+ K(P)
    ,
    \\
    \label{ChannelU}
    \sred_{x\sim P}\kwad{K(x) +\log P(x)} &\ge 0
    .
  \end{align}
\end{theorem}

\begin{theorem}
  \label{theoKPBU}
  For a~computer-dependent constant $A$,
  \begin{align}
    \label{ShannonFanoBU}
    K({x}^n) +\log \bm{P}({x}^n) &\le A+ K(\bm{P})+ K(n)
    ,
    \\
    \label{ChannelBU}
    \sred_{\bm{x}\sim \bm{P}}\kwad{K({x}^n) +\log \bm{P}({x}^n)} &\ge 0
    .
  \end{align}
  Moreover, 
  \begin{align}
    \label{BarronBU}
    K({x}^n) +\log \bm{P}({x}^n) &> 0
  \end{align}
  $n$-ultimately for $\bm{P}$-almost all
  sequences $\bm{x}$.
\end{theorem}

There are several simple corollaries of Theorem \ref{theoKPBU}.

\begin{definition}[Barron random sequence] 
  A~sequence $\bm{x}$ will be called $\bm{P}$-Barron random if
  (\ref{BarronBU}) holds $n$-ultimately for
  $\bm{x}$. The set of such sequences will be denoted as
  $\mathcal{B}_{\bm{P}}$. 
\end{definition}
\begin{definition}[Bayesian code]
  The Bayesian code with respect to
  $(\klam{\bm{P}_{\bm{\theta}}},\bm{\pi},{c})$ is the mapping
  $C:\mathbb{X}^+\ni x\mapsto C(x)={c}(|x|)\bm{C}(x)\in
  \mathbb{Y}^+$, where ${c}:\mathbb{N}\rightarrow \mathbb{Y}^+$ is
  a~code for natural numbers, $\bm{C}(x)$ is the Shannon-Fano codeword
  for $x$ with respect to $\bm{P}(x)=\int \bm{P}_{\bm{\theta}}(x)
  d\bm{\pi}(\bm{\theta})$, $\klam{\bm{P}_{\bm{\theta}}:\bm{\theta}\in
    \bm{\Theta}}$ is a~family of probability measures, and $\bm{\pi}$
  is a~prior probability.
\end{definition}
\begin{corollary}
  \label{theoKPBUone}
  If the measure $\bm{P}$ and code ${c}:\mathbb{N}\rightarrow
  \mathbb{Y}^+$ are computable then the Bayesian code with respect to
  $(\klam{\bm{P}_{\bm{\theta}}},\bm{\pi},{c})$ is
  $(\mathcal{B}_{\bm{P}},|{c}(n)|+1)$-superuniversal.
\end{corollary}
\begin{proof}
  Of course, the hypothesis implies that $C$ is a~computable prefix
  code. We have $|C(x)|=|{c}(|x|)| +\ceil{-\log \bm{P}(x)}$. If
  (\ref{BarronBU}) holds then $|C({x}^n)|-K({x}^n)< |{c}(n)|+1$.
  So $C$ is $(\mathcal{X},|{c}(n)|+1)$-superuniversal.
\end{proof}

Barron randomness is a~refinement of a~better known concept of
algorithmic randomness of sequences. Let us recall that sequence
$\bm{x}$ is $\bm{P}$-Martin-L\"of random
if and only if
\begin{align}
  \label{MartinLofBU}
  K({x}^n)+\log \bm{P}({x}^n) \ge -c
\end{align}
for some $c\ge 0$ and all $n$ \cite[Definition 2.5.4 and Theorem
3.6.1]{LiVitanyi97}. 
Denote the set of these sequences as $\mathcal{L}_{\bm{P}}$. We have
$\mathcal{L}_{\bm{P}}\supset \mathcal{B}_{\bm{P}}$ so
$\bm{P}(\mathcal{L}_{\bm{P}})=\bm{P}(\mathcal{B}_{\bm{P}})=1$. If
$\bm{x}\in \mathcal{L}_{\bm{P}}\setminus\mathcal{B}_{\bm{P}}$,
however, the catch-up time $\CUT(\bm{x};C)$ is infinite for
$|C(x)|=|{c}(|x|)| +\ceil{-\log \bm{P}(x)}$.

In the next step we will interpret the set of Barron random sequences
$\mathcal{B}_{\bm{P}}$ as a~superset of sequences typical of certain
not necessarily computable measures $\bm{P}_{\bm{\theta}}$. 
\begin{corollary}
  \label{theoKPBUtwo}
  Consider a~probability measure of form $\bm{P}(x)=\int
  \bm{P}_{\bm{\theta}}(x) d\bm{\pi}(\bm{\theta})$ for any measurable
  parameterization
  $\bm{\Theta}\ni\bm{\theta}\mapsto\bm{P}_{\bm{\theta}}$ where both
  prior $\bm{\pi}$ and $\bm{P}_{\bm{\theta}}$ are probability
  measures. Equality $\bm{P}_{\bm{\theta}}(\mathcal{B}_{\bm{P}})=1$
  holds for $\bm{\pi}$-almost all $\bm{\theta}$.
\end{corollary}
\begin{proof}
  Let $\mathcal{G}_n:=
  \klam{\bm{\theta}\in\bm{\Theta}:\bm{P}_{\bm{\theta}}(\mathcal{B}_{\bm{P}})
    \ge 1-1/n}$.  By Theorem \ref{theoKPBU},
  $1=\bm{P}(\mathcal{B}_{\bm{P}})\le
  \bm{\pi}(\mathcal{G}_n)+\bm{\pi}(\bm{\Theta}\setminus\mathcal{G}_n)(1-1/n)=
  1-n^{-1}\bm{\pi}(\bm{\Theta}\setminus\mathcal{G}_n)$.  Thus
  $\bm{\pi}(\mathcal{G}_n)=1$. Finally, we appeal to
  $\sigma$-additivity of $\bm{\pi}$.  For $\mathcal{G}:=
  \klam{\bm{\theta}\in\bm{\Theta}:\bm{P}_{\bm{\theta}}(\mathcal{B}_{\bm{P}})
    = 1}=\bigcap_n \mathcal{G}_n$ we obtain
  $\bm{\pi}(\mathcal{G})=\inf_n \bm{\pi}(\mathcal{G}_n)=1$.
\end{proof}

Corollaries \ref{theoKPBUone} and \ref{theoKPBUtwo} demonstrate that
the \emph{ultimate} redundancy of a~Bayesian code is nearly optimal
when compared with any computable code on data typical of
noncomputable parameterized measures $\bm{P}_{\bm{\theta}}$.  This
statement holds for any imaginable statistical model
$\klam{\bm{P}_{\bm{\theta}}:\bm{\theta}\in \bm{\Theta}}$.
Computability of the Bayesian code is the only restriction and the
only caveat is that $\bm{P}_{\bm{\theta}}(\mathcal{B}_{\bm{P}})=1$
holds for prior-almost all parameters $\bm{\theta}\in \bm{\Theta}$
rather than for all of them.

\begin{example}[a code for ``almost all'' distributions]
  \label{exBayesianCode}
  This example stems from the observation that we can encode any
  probability measure on $\mathbb{X}^\infty$ with a~single infinite
  sequence $\bm{\theta}={\theta}_1{\theta}_2{\theta}_3...$ over
  the alphabet $\mathbb{Y}=\klam{0,1,...,D-1}\ni
  \theta_n$.\footnote{One can also put $\bm{\theta}=\sum_{k=1}^\infty
    \theta_k D^{-k}$ since the set of real numbers having two
    different $D$-ary expansions is negligible.}

  For simplicity let the input alphabet be the set of natural numbers,
  $\mathbb{X}:=\mathbb{N}$. The link between $\bm{\theta}$ and
  a~measure $\bm{P}_{\bm{\theta}}$ will be established by imposing
  equality $\bm{P}_{\bm{\theta}}(x^n)=$
  \begin{align}
    \okra{
      \bm{P}_{\bm{\theta}}(x^{n-1})
      -\sum_{y\in\mathbb{N}:y<x_n} \bm{P}_{\bm{\theta}}(x^{n-1}y)
    }
    \cdot
    \sum_{k=1}^\infty \theta_{\phi(x^n,k)} D^{-k}
    \label{GenParam}
    ,
  \end{align}
  where $\bm{P}(\lambda)=1$ for the empty word and a~bijection $\phi:
  \mathbb{N}^+\times \mathbb{N}\rightarrow\mathbb{N}$ is used. It is
  easy to see that $\bm{P}_{\bm{\theta}}$ is a~probability measure on
  $\mathbb{X}^\infty$ for each $\bm{\theta}$. Conversely, each
  probability measure on $\mathbb{X}^\infty$ equals
  $\bm{P}_{\bm{\theta}}$ for at least one $\bm{\theta}$.

  Let the prior be the uniform iid measure
  $\bm{\pi}(\theta^m):=D^{-m}$ for
  ${\theta}^m:={\theta}_1{\theta}_2...{\theta}_m$. The Bayesian
  measure $\bm{P}(x)=\int \bm{P}_{\bm{\theta}}(x)
  d\bm{\pi}(\bm{\theta})$ is computable.  Consequently, the Bayesian
  code with respect to $(\klam{\bm{P}_{\bm{\theta}}},\bm{\pi},{c})$ is
  computable and $(\mathcal{B}_{\bm{P}},|{c}(n)|+1)$-superuniversal.
\end{example}

Whereas parameterization (\ref{GenParam}) is general, the measure
$\bm{P}$ introduced in this example equals simply
$\log_2 \bm{P}(x^n)=-\textstyle\sum_{i=1}^n x_i
$.
Although $\bm{P}_{\bm{\theta}}(\mathcal{B}_{\bm{P}})=1$ for
$\bm{\pi}$-almost all $\bm{\theta}$, the Bayesian code with respect to
this $\bm{P}$ is suboptimal for stationary measures different to
$\bm{P}$.

\section{Conclusion}
\label{secConclusion}

We hope that our simple insights may be used in future research to
better characterize several paradoxical phenomena that have haunted
the emerging MDL statistics. These phenomena  are:
nonexistence of universal redundancy rates, superefficient
compression/estimation, converging and diverging Bayesian predictors,
and various ``catch-up'' phenomena.

It is important to understand for which particular parameters the
claim of Corollary \ref{theoKPBUtwo} holds or fails.  Inspired by
\cite{Vovk09} and \cite{VitanyiLi00}, we have started contemplating
the following problem:
\begin{question} 
  \label{quesRandomParGen}
  Consider a~computable measure $\bm{P}(x)=\int
  \bm{P}_{\bm{\theta}}(x) d\bm{\pi}(\bm{\theta})$, where parameter
  values $\bm{\theta}$ are infinite sequences as well.  Let
  $\mathcal{X}$ be the set of sequences for which the Bayesian code
  with respect to $(\klam{\bm{P}_{\bm{\theta}}},\bm{\pi},{c})$ is
  $(\mathcal{X},|{c}(n)|+\mathcal{E}(n))$-universal.  What does
  $\bm{P}_{\bm{\theta}}(\mathcal{X})$ equal for $\bm{\theta}$ that (i)
  are algorithmically random, or (ii) exhibit a~deficiency of
  algorithmic randomness (e.g. they are computable)?
\end{question}
We have already proved that $\bm{P}_{\bm{\theta}}(\mathcal{X})=1$ for
(i) whereas $\bm{P}_{\bm{\theta}}(\mathcal{X})=0$ for (ii)
under some natural conditions, e.g., for exponential iid distributions
$\bm{P}_{\bm{\theta}}$.

The second group of interesting open problems concerns catch-up times.
Can we know the catch-up times approximately?  How can we use this
knowledge to verify or to falsify a~statistical model for concrete
data of limited length?

\section{Acknowledgement}

We express our thanks to Peter Gr\"unwald for interesting and
enjoyable discussions which stimulated us to write this paper. The
research reported in this work was supported in part by the IST
Program of the European Community, under the PASCAL II Network of
Excellence, IST-2002-506778, and done on the author's leave from the
Institute of Computer Science, Polish Academy of Sciences.

\bibliography{0-journals-abbrv,0-publishers-abbrv,ai,ql,mine,math,tcs,books,nlp}

\begin{thebibliography}{10}
\providecommand{\url}[1]{#1}
\csname url@rmstyle\endcsname
\providecommand{\newblock}{\relax}
\providecommand{\bibinfo}[2]{#2}
\providecommand\BIBentrySTDinterwordspacing{\spaceskip=0pt\relax}
\providecommand\BIBentryALTinterwordstretchfactor{4}
\providecommand\BIBentryALTinterwordspacing{\spaceskip=\fontdimen2\font plus
\BIBentryALTinterwordstretchfactor\fontdimen3\font minus
  \fontdimen4\font\relax}
\providecommand\BIBforeignlanguage[2]{{%
\expandafter\ifx\csname l@#1\endcsname\relax
\typeout{** WARNING: IEEEtran.bst: No hyphenation pattern has been}%
\typeout{** loaded for the language `#1'. Using the pattern for}%
\typeout{** the default language instead.}%
\else
\language=\csname l@#1\endcsname
\fi
#2}}

\bibitem{Dawid84}
A.~Dawid, ``Present position and potential developments: Some personal views,
  statistical theory, the prequential approach,'' \emph{J.\ Roy.\ Statist.\
  Soc.\ A}, vol. 147, pp. 278--292, 1984.

\bibitem{Rissanen84}
J.~Rissanen, ``Universal coding, information, prediction, and estimation,''
  \emph{IEEE Trans.\ Inform.\ Theor.}, vol.~30, pp. 629--636, 1984.

\bibitem{Grunwald07}
P.~D. Gr\"unwald, \emph{The Minimum Description Length Principle}.\hskip 1em
  plus 0.5em minus 0.4em\relax The MIT Press, 2007.

\bibitem{VitanyiLi00}
P.~Vit\'anyi and M.~Li, ``Minimum description length induction, {Bayesianism}
  and {Kolmogorov} complexity,'' \emph{IEEE Trans.\ Inform.\ Theor.}, vol.~46,
  pp. 446--464, 2000.

\bibitem{GacsTrompVitanyi01}
P.~G\'acs, J.~Tromp, and P.~M.~B. Vit\'anyi, ``Algorithmic statistics,''
  \emph{IEEE Trans.\ Inform.\ Theor.}, vol.~47, pp. 2443--2463, 2001.

\bibitem{BarronRissanenYu98}
A.~Barron, J.~Rissanen, and B.~Yu, ``The minimum description length principle
  in coding and modeling,'' \emph{IEEE Trans.\ Inform.\ Theor.}, vol.~44, pp.
  2743--2760, 1998.

\bibitem{ErvenGrunwaldRooij08}
T.~van Erven, P.~Grunwald, and S.~de~Rooij, ``Catching up faster by switching
  sooner: A prequential solution to the {AIC-BIC} dilemma,'' 2008,
  \url{http://arxiv.org/abs/0807.1005}.

\bibitem{Debowski08c}
{\L}.~D\k{e}bowski, ``On the vocabulary of grammar-based codes and the logical
  consistency of texts,'' 2008, e-print: \url{http://arxiv.org/abs/0810.3125}.
  Submitted to IEEE Transactions on Information Theory. In open review.

\bibitem{LiVitanyi97}
M.~Li and P.~M.~B. Vit\'anyi, \emph{An Introduction to {Kolmogorov} Complexity
  and Its Applications, 2nd ed.}\hskip 1em plus 0.5em minus 0.4em\relax
  Springer, 1997.

\bibitem{Topsoe79}
F.~Tops{\o}e, ``Information theoretical optimization techniques,''
  \emph{Kybernetika}, vol.~15, pp. 8--27, 1979.

\bibitem{Shtarkov87en2}
Y.~M. Shtarkov, ``Universal sequential coding of single messages,''
  \emph{Probl.\ Inform.\ Transm.}, vol. 23(2), pp. 3--17, 1987.

\bibitem{Haussler97}
D.~Haussler, ``A general minimax result for relative entropy,'' \emph{IEEE
  Trans.\ Inform.\ Theor.}, vol.~43, pp. 1276--1280, 1997.

\bibitem{Shields93}
P.~C. Shields, ``Universal redundancy rates don't exist,'' \emph{IEEE Trans.\
  Inform.\ Theor.}, vol. IT-39, pp. 520--524, 1993.

\bibitem{LeCam53}
L.~{Le Cam}, ``On sets of parameter points where it is possible to achieve
  superefficiency of estimates,'' \emph{Ann.\ Math.\ Statist.}, vol.~23, p.
  148, 1953.

\bibitem{Vovk91en2}
V.~Vovk, ``Asymptotic efficiency of estimators: an algorithmic approach,''
  \emph{Theor.\ Probab.\ Appl.}, vol.~36, pp. 329--343, 1991.

\bibitem{BarronHengartner98}
A.~Barron and N.~Hengartner, ``Information theory and superefficiency,''
  \emph{Ann.\ Statist.}, vol.~26, pp. 1800--1825, 1998.

\bibitem{Vovk09}
V.~Vovk, ``Superefficiency from the vantage point of computability,'' 2009,
  submitted to the Statistical Science.

\bibitem{ZivLempel77}
J.~Ziv and A.~Lempel, ``A universal algorithm for sequential data
  compression,'' \emph{IEEE Trans.\ Inform.\ Theor.}, vol.~23, pp. 337--343,
  1977.

\bibitem{NeuhoffShields98}
D.~Neuhoff and P.~C. Shields, ``Simplistic universal coding,'' \emph{IEEE
  Trans.\ Inform.\ Theor.}, vol. IT-44, pp. 778--781, 1998.

\bibitem{Kieffer78}
J.~Kieffer, ``A unified approach to weak universal source coding,'' \emph{IEEE
  Trans.\ Inform.\ Theor.}, vol.~24, pp. 674--682, 1978.

\bibitem{Elias75}
P.~Elias, ``Universal codeword sets and representations for the integers,''
  \emph{IEEE Trans.\ Inform.\ Theor.}, vol.~21, pp. 194--203, 1975.

\bibitem{CoverThomas91}
T.~M. Cover and J.~A. Thomas, \emph{Elements of Information Theory}.\hskip 1em
  plus 0.5em minus 0.4em\relax Wiley, 1991.

\bibitem{Barron85b}
A.~R. Barron, ``Logically smooth density estimation,'' Ph.D. dissertation,
  Stanford University, 1985.

\end{thebibliography}

\end{document}